\documentclass{article} 
\usepackage{iclr2025_conference,times}
\usepackage{graphicx}

\usepackage{amsmath,amsfonts,bm}

\def\eqref#1{equation~\ref{#1}}

\def\1{\bm{1}}

\DeclareMathAlphabet{\mathsfit}{\encodingdefault}{\sfdefault}{m}{sl}
\SetMathAlphabet{\mathsfit}{bold}{\encodingdefault}{\sfdefault}{bx}{n}

\def\gD{{\mathcal{D}}}

\def\gL{{\mathcal{L}}}

\newcommand{\E}{\mathbb{E}}

\newcommand{\R}{\mathbb{R}}

\usepackage{amsthm}
\newtheorem{theorem}{Theorem}

\usepackage{hyperref}
\usepackage{url}

\title{Understanding SOAP from the Perspective of Gradient Whitening}

\author{Yanqing Lu, Letao Wang \& Jinbo Liu \thanks{Correspondence to Yanqing Lu: $<$\texttt{\href{mailto:ylu62702@usc.edu}{ylu62702@usc.edu}}$>$.}\\
Department of Computer Science\\
University of Southern California\\
}

\iclrfinalcopy 
\begin{document}

\maketitle


\begin{abstract}
Shampoo with Adam in the Preconditioner’s eigenbasis (SOAP) has recently emerged as a promising optimization algorithm for neural network training, achieving superior training efficiency over both Adam and Shampoo in language modeling tasks. In this work, we analyze Adam, Shampoo, and SOAP from the perspective of gradient whitening, interpreting their preconditioners as approximations to the whitening matrix, which captures second-order curvature information. We further establish a theoretical equivalence between idealized versions of SOAP and Shampoo under the Kronecker product assumption. To empirically evaluate these insights, we reproduce the language modeling experiments using nanoGPT and grayscale image colorization. Our results show that SOAP exhibits similar convergence rate as Shampoo, and no significant advantage over both Adam and Shampoo in the final loss achieved, which aligns with their equivalence in theory.
\end{abstract}

\section{Introduction}
As deep learning models continue to scale in size and complexity, improving optimization efficiency has become a critical concern. Recent advances have introduced ShampoO with Adam in the Preconditioner’s eigenbasis (SOAP) \citep{vyas2024soap}, an optimization algorithm that runs Adam in the rotated parameter space provided by Shampoo. Empirical results in language modeling tasks demonstrate that SOAP outperforms both Adam \citep{kingma2014adam} and Shampoo \citep{gupta2018shampoo} in terms of training efficiency. To better understand the design of SOAP, we analyze it through the lens of \textit{gradient whitening}. In this view, the preconditioners used by Adam, Shampoo, and SOAP are interpreted as different approximations of the whitening matrix, which captures second-order curvature information of model parameters during training. Crucially, we establish a theoretical connection between SOAP and Shampoo. Building on these insights, we reproduce the language modeling experiment using nanoGPT to benchmark SOAP against Adam and Shampoo, and further extend our evaluation to computer vision tasks to assess SOAP’s effectiveness in optimizing convolutional neural networks.

\section{Notation and Background}
Neural network layers typically have matrix-shaped weights. We denote the weight matrix and the gradient of loss with respect to it by $W\in \R^{m\times n}$ and $G\in \R^{m\times n}$, respectively. Let $g\in \R^{mn}$ be the vectorized gradient, the update direction of adaptive gradient descent algorithms often takes the form of $-H^{-p} g$, where $H\in\R^{mn\times mn}$ is referred to as the preconditioner, and $p\in[0,1]$ is some constant that controls the degree of adaptation.

\textbf{Adam} \citep{kingma2014adam} is a widely used first-order adaptive algorithm that employs a diagonal preconditioner. It maintains exponential moving averages of both the gradient (denoted as $M$) and its element-wise square (denoted as $V$). Given learning rate $\eta$, decay rates $\beta_1$, $\beta_2$ and the current gradient $G$, the update rule of Adam is
\begin{align}\label{eq:adam}
    M \leftarrow \beta_1M + (1-\beta_1)G;\quad V \leftarrow \beta_2V + (1-\beta_2)G^2;\quad W \leftarrow W - \eta M / \sqrt{V}.
\end{align}

\textbf{Shampoo} \citep{gupta2018shampoo} is a second-order optimization method maintaining two preconditioner matrices $L\in\R^{m\times m}$ and $R\in\R^{n\times n}$, which can be viewed as a Kronecker factorization of a full-matrix preconditioner of shape $mn\times mn$. We adopt a modified version of the original Shampoo as suggested in recent works \citep{shi2023distributed, vyas2024soap}. For simplicity, we also omit the exponential moving average (EMA) operations of the preconditioners. The resulting update step is given by
\begin{align}\label{eq:shampoo}
    L \leftarrow GG^T;\quad R \leftarrow G^TG;\quad W \leftarrow W - \eta L^{-\frac{1}{2}}GR^{-\frac{1}{2}}/\text{Trace}(L)^{-\frac{1}{2}}.
\end{align}

\textbf{SOAP} \citep{vyas2024soap} is a second-order optimization method that aims at enhancing the adaptivity of Shampoo in practice. It defines a sequence of operations including transforming the current gradient using the eigenvector matrices of Shampoo’s preconditioners, applying an Adam update to the transformed gradient, and then projecting the result back to the original space to update the weights. Let $Q_L, Q_R\in\R^{m\times n}$ denote the eigenvectors of the preconditioners $L$ and $R$, and let \texttt{Adam}$(G')$ denote the Adam-adapted gradient of $G'$. The overall weight update is given by 
\begin{align}\label{eq:soap}
    G' &\leftarrow Q_L^TGQ_R;\quad N' \leftarrow \texttt{Adam}(G');\quad N \leftarrow Q_LN'Q_R^T;\quad W \leftarrow W - \eta N.
\end{align}

\section{Gradient Whitening}\label{sec:whitening}
Recall that we define $g\in \R^{mn}$ as the vectorized gradient of the training loss with respect to the weight matrix. $g$ can be viewed a random vector where the randomness originates from the input features and labels. Let $\Sigma = \E[gg^T]$ denote the covariance matrix of the gradient vector, and define $g'=\Sigma^{-1/2}g$. $\Sigma$ is a \textbf{whitening matrix}, since one can show that $\E[g'g'^T]=I$. This transformation, known as \textit{gradient whitening} \citep{vyas2024soapmuon}, standardizes and decorrelates the gradient, which intuitively leads to improved convergence during optimization. Furthermore, there exists a theoretical connection between the covariance matrix and the Hessian.

\textbf{Whitening matrix as the Gauss-Newton component of the Hessian.} Let $\gD$ denote the training distribution, and let $g_{x,y}\in\R^{mn}$ represent the vectorized gradient for a single training example $(x,y)\sim \gD$. We also denote the vectorized weight by $w\in\R^{mn}$. Given a neural network output $f(x)$, the loss function is written as $\gL(f(x), y)$. Then the Hessian of the expected loss with respect to $w$ can be decomposed as follows:
\begin{align}\label{eq:hessian}
    \E_{(x,y)\sim \gD} \left[\frac{\partial^2\gL}{\partial w^2}\right] = \E_{(x,y)\sim \gD} \left[\frac{\partial f}{\partial w}^T \frac{\partial^2\gL}{\partial f^2} \frac{\partial f}{\partial w}\right] + \E_{(x,y)\sim \gD} \left[\frac{\partial \gL}{\partial f} \frac{\partial^2 f}{\partial w^2}\right].
\end{align}

The first term on the right-hand side of Equation~\ref{eq:hessian} is known as the Gauss-Newton component of the Hessian, denoted by $H_{\text{GN}}$. This component has been shown to closely resemble the full Hessian during neural network training in terms of their eigenspectra \citep{sankar2021deeper}. Moreover, in the case of cross-entropy loss, the Gauss-Newton component is referred to as the Fisher Information Matrix, which coincides with the whitening matrix \citep{morwani2024new}. Consequently, we have
\begin{align}\label{eq:gn}
    H_{\text{GN}} = \E_{(x,y)\sim \gD} \left[\frac{\partial f}{\partial w}^T \frac{\partial^2\gL}{\partial f^2} \frac{\partial f}{\partial w}\right] = \E_{x, s\sim f(x)}[g_{x,s}g_{x,s}^T] = \Sigma.
\end{align}
Here we explicitly write $x, s\sim f(x)$ to emphasize that the label is sampled from the model’s output distribution, rather than being the true label from the training set, where the latter is typically used in practice. This distinction is analyzed in the context of Shampoo by \citet{morwani2024new}.

\section{Approximating the Whitening Matrix}\label{sec:approx}
Using the whitening matrix as a preconditioner in adaptive optimization algorithms is a promising direction, as it contains rich second-order information about the curvature of the loss landscape, as discussed in Section~\ref{sec:whitening}. However, applying the full whitening matrix is computationally expensive and does not scale well to large deep learning models with billions of parameters. With the objective of scalable yet effective preconditioning, several adaptive methods can be interpreted as using approximations to the whitening matrix, each based on different structural assumptions. 

\textbf{Adam as a diagonal version of whitening.} To better understand Adam in the context of gradient whitening, we rewrite its update of weight from Equation~\ref{eq:adam} in vectorized form. Let $m$ denote the vectorized first-order momentum, and replace the EMA of the squared gradients with the expectation over training distribution. This leads to the following update step:
\begin{align}
    w \leftarrow w - \eta H_{\text{Adam}}^{-\frac{1}{2}} m,
\end{align}
where $H_{\text{Adam}}\in\R^{mn\times mn}$ is a diagonal matrix with entries $(H_{\text{Adam}})_{i,i} = \E(g_i^2)$ for $i=1,2,...,mn$. Observe that $H_{\text{Adam}} = diag(\E[gg^T]) = diag(\Sigma)$. Thus, Adam can be viewed as a gradient whitening optimizer with momentum under the assumption that $\Sigma$ is diagonal, or equivalently the gradients are uncorrelated.

\textbf{Shampoo as the optimal Kronecker product approximation.} Using the following properties of Kronecker product: (1) $(A\otimes B)\text{vec}(G) = \text{vec}(BGA^T)$; (2) $(A\otimes B)^p = A^p \otimes B^p$ \citep{henderson1981vec}, the vectorized version of Shampoo's weight update can be expressed as 
\begin{align}\label{eq:shampoo_vec}
    w \leftarrow w - \eta H_{\text{Shampoo}}^{-\frac{1}{2}}g = w - \eta ((R\otimes L)/\text{Trace}(L))^{-\frac{1}{2}} g.
\end{align}
Assuming that the whitening matrix $\Sigma$ can be factorized as a Kronecker product, its explicit form based on the analysis of \citet{morwani2024new} is given by:
\begin{align}\label{eq:kronecker}
    \Sigma = (\E[G^TG]\otimes \E[GG^T])/\text{Trace}(\E[GG^T]) \approx H_{\text{Shampoo}}.
\end{align}
From Equation~\ref{eq:gn} we know that $\Sigma$ corresponds to the Gauss-Newton component of the Hessian $H_{\text{GN}}$ when the expectation is taken over the training distribution of $x$ with sampled label. In practice, the preconditioner of Shampoo differs from the middle term of Equation~\ref{eq:kronecker} due to the following three modifications: (1) the per-input gradient is replaced by the \textit{batch gradient}, as models are typically trained using mini-batch sampling; (2) sampled labels are replaced with real labels; (3) EMAs of $L, R$, and $G$ are maintained, rather than the exact expectations over the training distribution. 

While these approximations may introduce some deviation from the theoretical whitening matrix, prior work shows that their impact on optimization performance is often negligible \citep{grosse2021adaptive, osawa2023asdl}. Therefore, Shampoo can still be viewed as a gradient whitening optimizer with the Kronecker product assumption.

\section{Equivalence between SOAP and Shampoo}
An important motivation behind SOAP is the equivalence between the idealized Shampoo and running idealized Adafactor \citep{shazeer2018adafactor} in the eigenbasis provided by Shampoo's preconditioners \citep{vyas2024soap}. While SOAP only replaces Adafactor with Adam in the above algorithm, we state that SOAP and Shampoo are also equivalent in their idealized forms when the whitening matrix admits a Kronecker product structure. 

In the idealized version of SOAP and Shampoo, all EMAs are replaced by expectations over the training distribution with sampled labels, except the EMA of the gradient (i.e., the first-order momentum). For brevity, we omit this component in the following analysis, as it does not affect the equivalence.

\begin{theorem}
    Idealized SOAP and idealized Shampoo are equivalent by assuming that the gradient whitening matrix $\Sigma\in\R^{mn\times mn}$ is a Kronecker product of $L\in\R^{m\times m}$ and $R\in\R^{n\times n}$.
\end{theorem}

\begin{proof}
    Recall that the whitening matrix is defined as the covariance matrix of vectorized gradients. From Equation \ref{eq:kronecker}, we have $\Sigma=\E[gg^T]=(R\otimes L)/\text{Trace}(L)$. By performing eigendecomposition on Shampoo's preconditioners, we can express $L$ and $R$ as
    \begin{align*}
        L = Q_L\Lambda_L Q_L^T, R = Q_R\Lambda_R Q_R^T.
    \end{align*}
    Let $g'\in\R^{mn}$ denote the vectorized rotated gradient $G'=Q_L^T G Q_R$ at the first step of SOAP, $g'=\text{vec}(G') = (Q_R^T \otimes Q_L^T)g$. Similarly, we can express the remaining vectorized update steps in Equation~\ref{eq:soap} as $n'=\text{vec}(N')=diag(\E[g'g'^T])^{-1/2}g'$, and $n=\text{vec}(N) = (Q_R\otimes Q_L)n'$, where $n$ is the final vectorized update direction of SOAP.
    
    We first show that the covariance matrix of the rotated gradients is diagonal.
    \begin{align*}
        \E[g'g'^T] &= (Q_R^T\otimes Q_L^T)\E[gg^T](Q_R\otimes Q_L)\\
        &= (Q_R^T\otimes Q_L^T)((R\otimes L)/\text{Trace}(L))(Q_R\otimes Q_L)\\
        &= (Q_R^T R Q_R)\otimes (Q_L^T LQ_L)/\text{Trace}(L)\\
        &= \Lambda_R \otimes \Lambda_L/\text{Trace}(L) \quad \quad \quad \quad \quad \text{(diagonal since $\Lambda_R$ and $\Lambda_L$ are both diagonal)}
    \end{align*}
    It follows that $n'=(\E[g'g'^T])^{-1/2}g'=(\Lambda_R \otimes \Lambda_L/\text{Trace}(L))^{-1/2}g'$, and the final update direction of SOAP becomes $n=(Q_R\otimes Q_L) (\Lambda_R \otimes \Lambda_L/\text{Trace}(L))^{-1/2} (Q_R^T \otimes Q_L^T)g$.

    Now we rewrite the Shampoo update direction in Equation~\ref{eq:shampoo_vec} to prove their equivalence.
    \begin{align*}
        ((R\otimes L)/\text{Trace}(L))^{-\frac{1}{2}} g &= (((Q_R\Lambda_R Q_R^T)\otimes (Q_L\Lambda_L Q_L^T)/\text{Trace}(L))^{-\frac{1}{2}} g\\
        &= ((Q_R\otimes Q_L)(\Lambda_R\otimes \Lambda_L)(Q_R^T\otimes Q_L^T)/\text{Trace}(L))^{-\frac{1}{2}}g\\
        &= (Q_R\otimes Q_L)(\Lambda_R\otimes \Lambda_L/\text{Trace}(L))^{-\frac{1}{2}}(Q_R^T\otimes Q_L^T))g 
    \end{align*}
\end{proof}

While the idealized forms of these two algorithms are equivalent under certain assumptions, their practical performance differs due to the modifications discussed in the Shampoo part of Section~\ref{sec:approx}. In addition, differences in training data distribution, model architectures, and efficiency-related implementation trade-offs contribute further to the discrepancy between SOAP and Shampoo.

\section{Experiments}

We reproduced Adam, Shampoo and SOAP respectively, on a vision task (grayscale image colorization) and a language task (nanoGPT \cite{Karpathy2022}). We measured the convergence rate and the final loss achieved of the 3 optimizers, and the effect of precondition frequency on Shampoo and SOAP.

We used a modified ResNet and CNN architecture for the image colorization task \cite{colorization}. Our model has roughly 12 million parameters. We trained our model using MSE loss, on 4258 images selected from the MIT Places 365 dataset \cite{places}. We converted these images to Lab format and trained our model to predict the ab channel using the L channel.

We used nanoGPT for the language task, which employs a Transformer architecture consisting of 6 layers, 6 attention heads, a 384-dimensional embedding, and a context window of 256 characters, resulting in approximately 10.6 million parameters. The model is trained on the standard "tiny-Shakespeare" dataset, which contains approximately 1.1 million characters and a vocabulary size of 65. This corpus is divided into training and validation sets with a 90\%/10\% split. We used the standard cross-entropy loss (next-token prediction).

Training loss converges after one epoch for all runs in this experiment. We report the validation loss in the following sections.

\subsection{Convergence rate} 

\begin{figure}[t]
  \centering
  \includegraphics[width=0.6\textwidth]{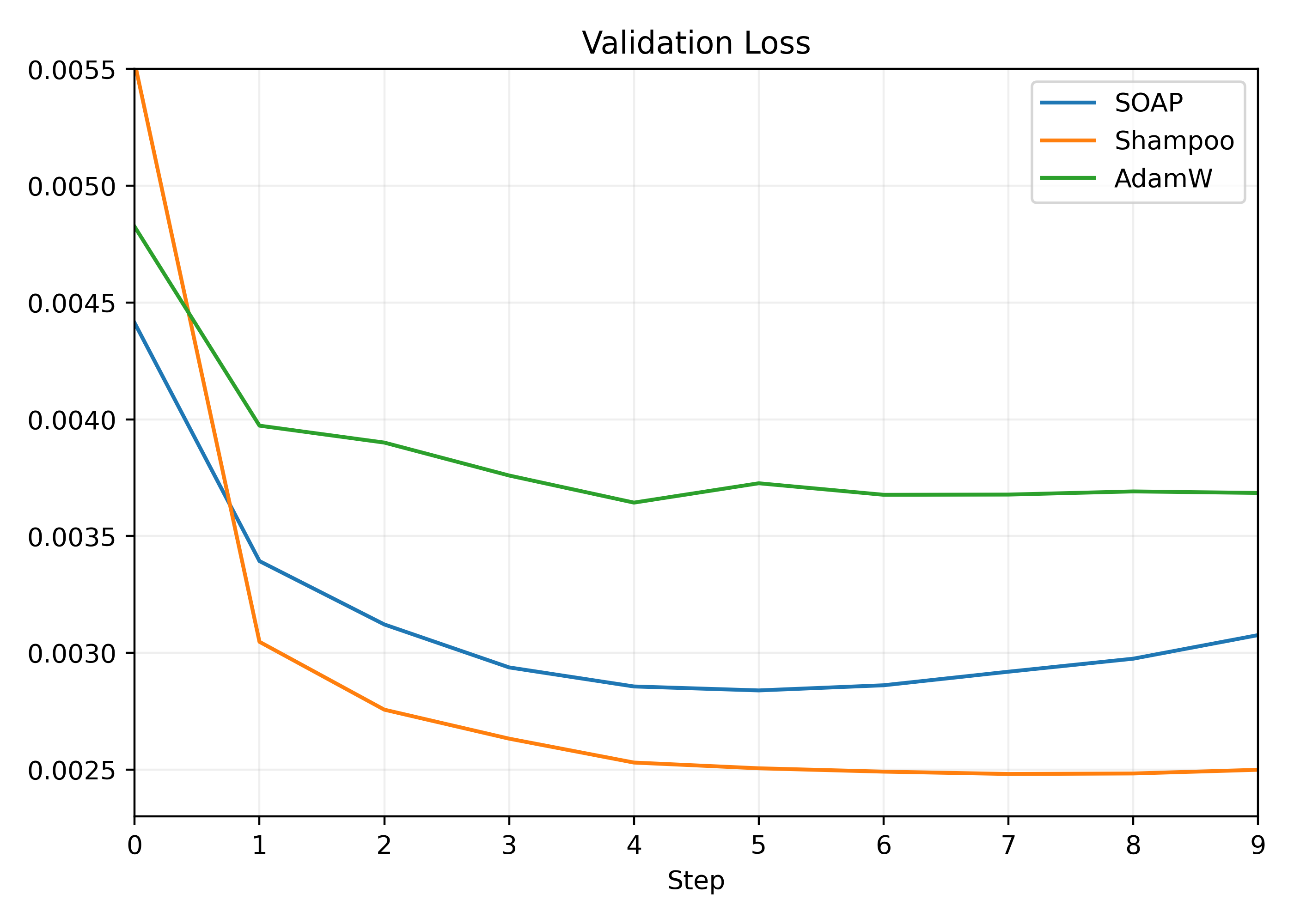} 
  \caption{Validation loss curves for Shampoo, SOAP, and AdamW optimizers on the image colorization task. Shampoo achieves the fastest convergence, while SOAP and AdamW converge more slowly and plateau at higher loss values. This highlights the advantage of second-order optimization in this setting.}
  \label{fig:convergence}
\end{figure}

In image colorization, as shown in Figure~\ref{fig:convergence} Shampoo appears to converge faster despite having a theoretical edge on SOAP. Adam, however, has the slowest convergence, in accordance with theory.

\begin{figure}[t] \label{nano_out}
    \centering  
    \includegraphics[width=\linewidth]{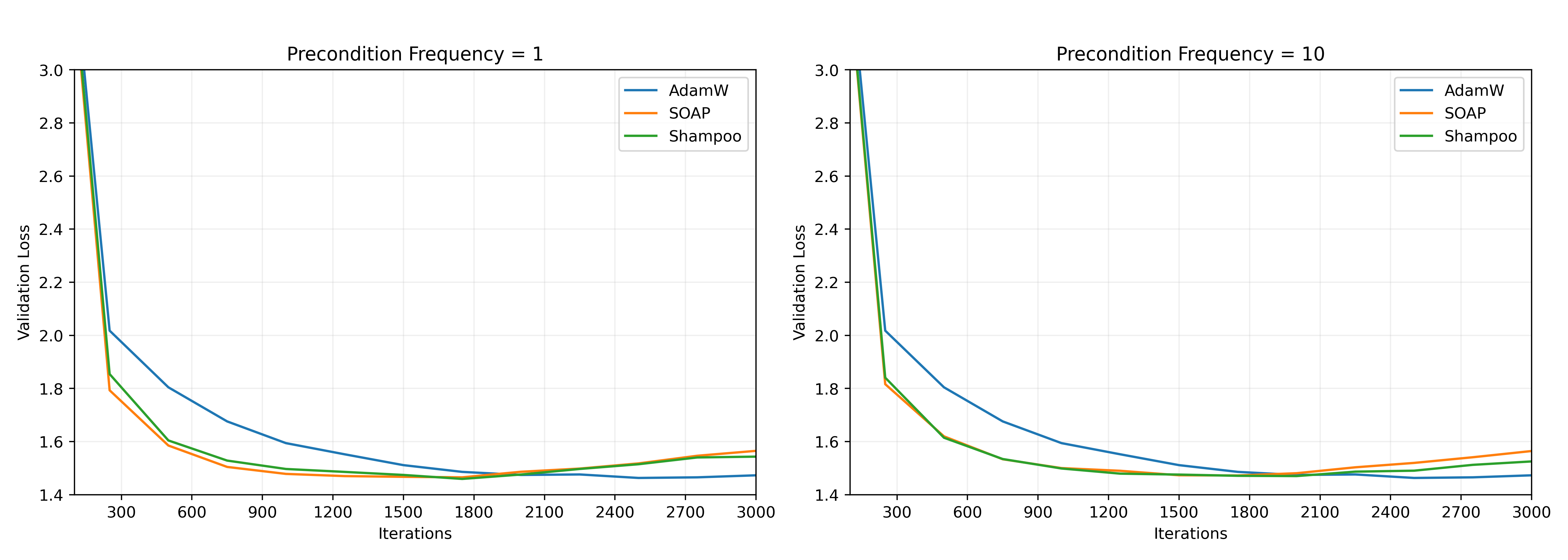}
    \caption{Validation loss for nanoGPT with different precondition frequencies}
    \label{fig:nano_out}
\end{figure}

In terms of the autoregressive language‑modeling task, we trained a character-level GPT (nanoGPT) on the tiny Shakespeare dataset with the three different optimizers. Figure~\ref{fig:nano_out} shows the validation loss during training process. AdamW consistently has the slowest convergence. In contrast, SOAP and Shampoo have almost identical convergence rates, both faster than AdamW, regardless of precondition frequency value.

\subsection{Final Loss}

\begin{figure}[t]
  \centering
  \includegraphics[width=0.6\textwidth]{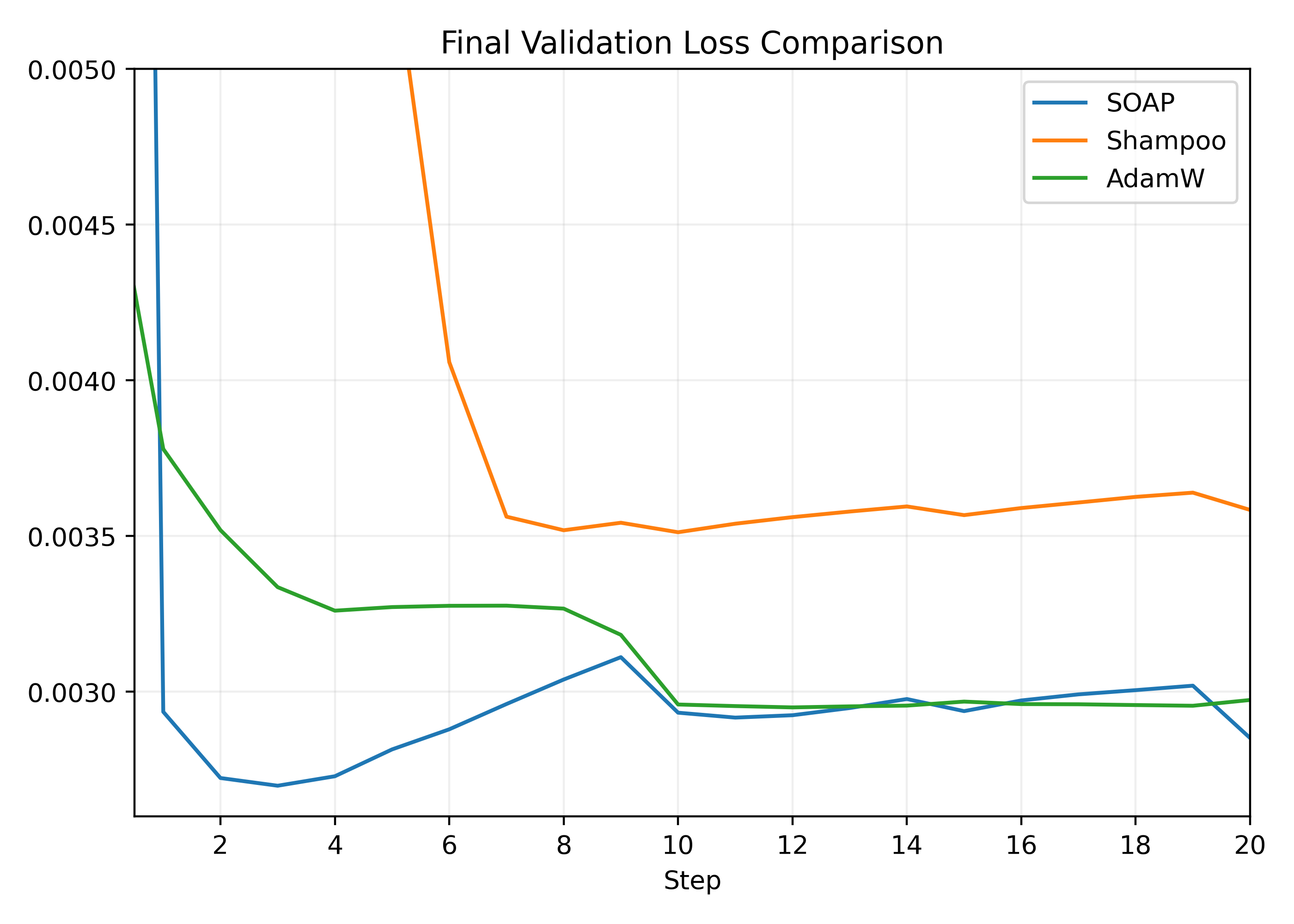} 
  \caption{Validation loss curves for Shampoo, AdamW, and SOAP optimizers on image colorization. Shampoo converges rapidly in early epochs but plateaus at a higher final loss, while AdamW and SOAP achieve lower final losses, with SOAP closely matching AdamW’s performance. All optimizers show loss stabilization by epoch 10.}
  \label{fig:final_loss}
\end{figure}

We used pixel-level MSE loss in image colorization. The loss converges after epoch 10. We trained to epoch 20 to verify convergence. Adam and SOAP both achieved a final pixel-level MSE loss of 0.0029, while Shampoo achieved a final loss of 0.0035. In this case, despite faster initial convergence, Shampoo failed to recover Adam's final loss, while SOAP reproduced it.

In the nanoGPT task, AdamW achieved a final validation loss of 1.4632. SOAP obtained final losses of 1.4654 and 1.4724 with precondition frequencies of 1 and 10, respectively, whereas Shampoo achieved losses of 1.4595 and 1.4703 at the same frequency. These results indicate that although Shampoo produced the best performance when the precondition frequency was set to 1, the overall differences in loss among the three optimization methods were relatively minor.

\subsection{Precondition Frequency}

\begin{figure}[t]
  \centering
  \includegraphics[width=0.6\textwidth]{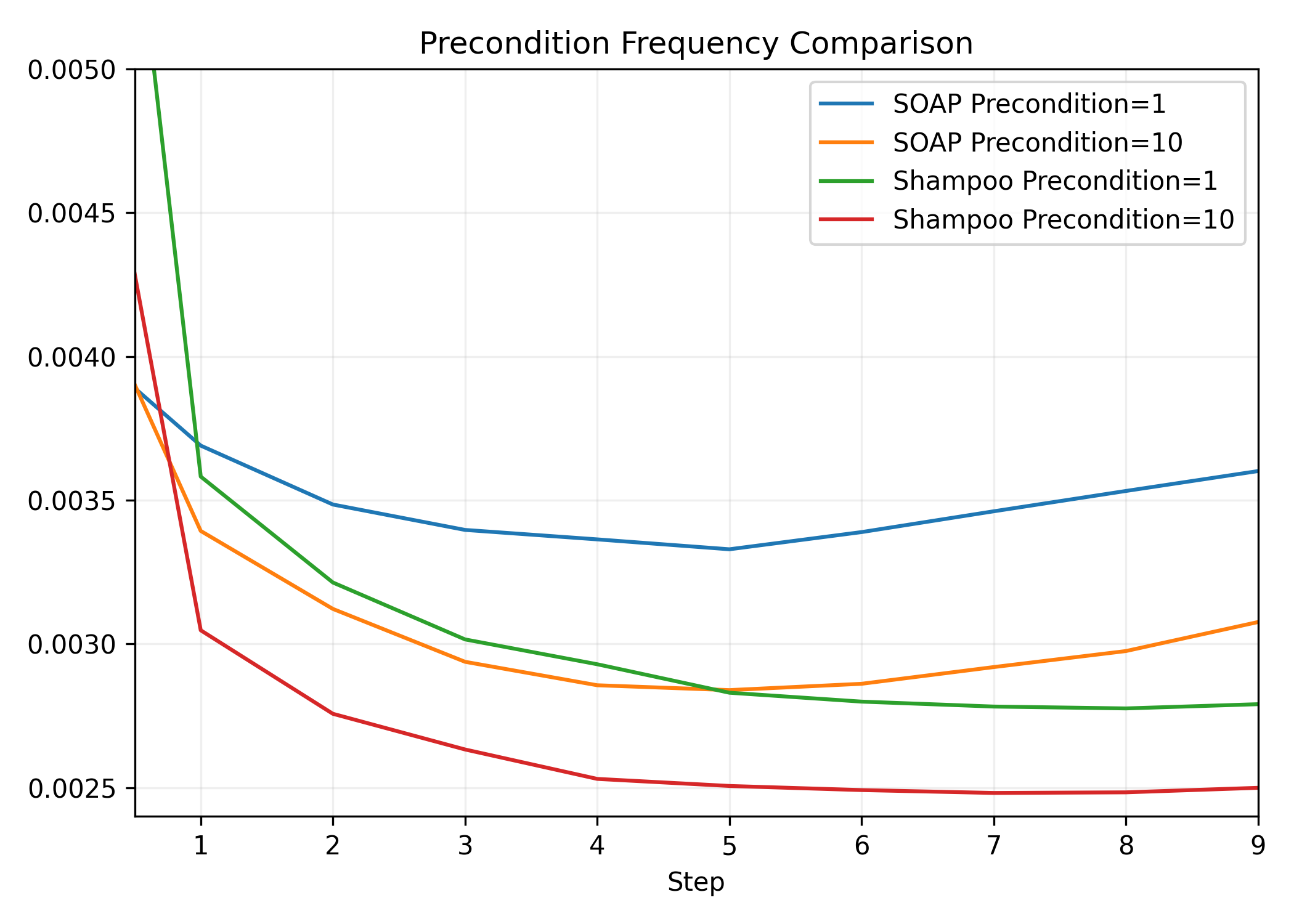} 
  \caption{Validation loss curves for Shampoo and SOAP optimizers on image colorization, with different preconditioner update frequencies. Updating the preconditioner every 10 steps leads to lower final loss and greater stability, especially for Shampoo, which outperforms SOAP across both frequencies.}
  \label{fig:precondition}
\end{figure}

Figure~\ref{fig:precondition} shows the comparison between different preconditioning frequency for Shampoo and SOAP on image colorization. We did not observe faster convergence with lower precondition frequency, possibly because noisy gradients lead to unstable curvature estimates, while the true curvature changes slowly and is already well-captured by EMA smoothing.

For nanoGPT, varying the precondition frequency did not result in significant differences. As shown in Figure~\ref{fig:nano_out}, the convergence rate improved only slightly at lower precondition frequency.

\section{Conclusion}
In this work, we examine SOAP as an optimization algorithm that leverages second-order information through whitening-based preconditioning. By analyzing SOAP, Shampoo, and Adam through the lens of gradient whitening, we show how each method approximates the whitening matrix under different structural assumptions. We highlight a theoretical equivalence between idealized SOAP and Shampoo when the whitening matrix has a Kronecker product structure, providing a deeper understanding of SOAP's design. Our empirical studies, conducted on both language modeling and image colorization tasks, demonstrate that while SOAP and Shampoo perform comparably, they do not present consistent advantages over Adam. These findings emphasize the broad potential of gradient whitening as a principled foundation for adaptive optimization, and also indicate that the effectiveness of second-order preconditioning may vary depending on the problem domain and model architecture.

\bibliography{iclr2025_conference}
\bibliographystyle{iclr2025_conference}


\end{document}